\documentclass[11pt]{article}

\usepackage{times}
\usepackage{color}
\usepackage{graphicx,amssymb,amsmath}
\usepackage{xspace}

\newcommand{\eat}[1] {}

\newcommand{\QED} {\hfill$\Box$}

\newtheorem{theorem}{\bf Theorem}[section]
\newtheorem{corollary}[theorem]{\bf Corollary}
\newtheorem{lemma}[theorem]{\bf Lemma}

\newtheorem{claim}[theorem]{\bf Claim}

\newcommand{\lp}{\mathrm{LP}}

\newcommand{\cp}{\mathrm{CP}}
\newcommand{\mts}{\mathrm{MTS}}
\newcommand{\mwu}{\mathrm{MWU}}
\newcommand{\wtxit} {\widetilde{x_{i,t}}}
\newcommand{\wtxt}{\widetilde{x_t}}

\newcommand{\omd}{{\mathrm{OMD}}}
\newcommand{\oco}{{\mathrm{OCO}}}
\newcommand{\leqs}{\leqslant}
\newcommand{\geqs}{\geqslant}
\newcommand{\zerola}{\mathrm{0LA}}
\newcommand{\onela}{\mathrm{1LA}}

\newcommand{\cL}{{\cal L}}

\newcommand{\bigo}{{\cal O}}

\newcommand{\eps}{\epsilon}

\newcommand{\opt}{\mathrm{OPT}}

\newcommand{\nocomma}{}

\newcommand{\nosymbol}{}
\newcommand{\tmmathbf}[1]{\ensuremath{\boldsymbol{#1}}}

\newcommand{\tmop}[1]{\ensuremath{\operatorname{#1}}}
\newcommand{\tmtextbf}[1]{{\bfseries{#1}}}
\newcommand{\tmtextit}[1]{{\itshape{#1}}}

\title{Fenchel Duals for Drifting Adversaries}

\author{Suman K Bera\thanks{Indian Institute of Technology Delhi. Email: \textrm{sumankalyanbera@gmail.com}.},
~ Anamitra R Choudhury\thanks{IBM Research, Delhi. Email: \textrm{anamchou@in.ibm.com}},
~ Syamantak Das\thanks{Indian Institute of Technology Delhi. Email: \textrm{sdas@cse.iitd.ac.in}.},
~ Sambuddha Roy\thanks{IBM Research, Delhi. Email: \textrm{shombuddho@gmail.com}}
~and~ Jayram S. Thatchachar\thanks{IBM Research, Almaden. Email: \textrm{jayram@us.ibm.com}. }}

\begin{document}

\maketitle

\begin{abstract}
We describe a primal-dual framework for the design and analysis of online convex optimization algorithms for {\em drifting regret}. Existing literature shows (nearly) optimal drifting regret bounds only for the $\ell_2$ and the $\ell_1$-norms. Our work provides a connection between these algorithms and the Online Mirror Descent ($\omd$) updates; one key insight that results from our work is that in order for these algorithms to succeed, it suffices to have the gradient of the regularizer to be bounded (in an appropriate norm). For situations (like for the $\ell_1$ norm) where the vanilla regularizer does not have this property, we have to {\em shift} the regularizer to ensure this. Thus, this helps explain the various updates presented in \cite{bansal10, buchbinder12}. We also consider the online variant of the problem with $1$-lookahead, and with movement costs in the $\ell_2$-norm. Our primal dual approach yields nearly optimal competitive ratios for this problem.
\end{abstract}

\section{Introduction}
\label{sec:intro}
The problem of online learning considers the task of a decision maker who needs to iteratively make decisions 
in the face of uncertainty. In each round or iteration, the decision maker picks a certain action from an action set (or feasible set). It is only after this, that the cost function (for this iteration) is revealed to the decision maker, and the decision maker then suffers the penalty indicated by the cost function on the specific choice made by him in this round. The overall objective of the decision maker is to {\em minimize} the {\em regret} where regret is the difference between the total cost of the decision maker and the  cost of an arbitrary {\em fixed} policy. 
Online convex optimization ($\oco$) is the specific instance of the above model wherein the incumbent cost functions are {\em convex}. The seminal paper of Zinkevich~\cite{zinkevich03} achieves (the optimal) 
 ${\bigo}(\sqrt{T})$ regret bounds against fixed adversaries, in the $\oco$ framework.

However, it has often been pointed out that the assumption of a {\em fixed} policy for the adversary may not be a valid one and it is too much to expect for a single policy to perform well. Recently there has been work in the direction of adversaries that may {\em drift} (i.e. change their policy). Thus, a generalization of the regret model is one in which the adversaries are allowed to drift to a certain extent. Prior work has considered variants of this theme, for instance {\em adaptive} regret~\cite{hazan09}, tracking the best expert~\cite{herbster98} or settings where the adversary may change her policy only a certain number of times ({\em shifting} regret; see \cite{bianchi12, crammer10}) etc. It may be noted that the model of regret is relatively {\em memoryless}: the cost in any iteration is dependent on {\em only} the current configuration of the decision maker.

Another emerging line of work has focused on settings where there are {\em states}; thus, the costs in any specific iteration depend on all of the decision maker's {\em past} configurations (see  \cite{buchbinder12}). Thus, the decision maker needs to compete against {\em arbitrary} policies, including the optimal offline policy, $\opt$, that has access to all of the cost functions, and where we can model states. This problem belongs to the realm of {\em competitive analysis} (cf. the text on online algorithms,~\cite{borodin98}). Since the offline optimum $\opt$ is significantly powerful in this model, achieving {\em additive} regret bounds is hard in general. Instead, the measure used is the {\em competitive ratio} that is inherently a {\em multiplicative} factor between the total cost incurred by the decision maker as against the cost of $\opt$.

In order to model states, the competitive analysis framework assumes a {\em movement cost} (also referred to as {\em switching costs}) between successive configurations of the decision maker. More significantly, this framework assumes $1$-lookahead (denoted as $\onela$): the decision maker knows the cost vector in the current iteration. In contrast, the framework of regret assumes $0$-lookahead (denoted as $\zerola$) in which the decision maker has to make the current prediction and only after that does she receive the cost vector for the current iteration. 


Since this framework involves {\em multiplicative} factors (instead of additive factors), it is necessary that  the cumulative cost incurred by the offline $\opt$ be {\em non-negative}. This implies two restrictions for $\onela$: one is that the individual cost functions are {\em non-negative} and secondly, the convex body in which the points are chosen has to lie entirely in the {\em non-negative} orthant of $\mathbb{R}^n$.

Recently, researchers have looked at the relation between these two frameworks. Work by Andrew et al.~\cite{barman13} show that a {\em single} algorithm {\em cannot} simultaneously achieve low regret (in the $\zerola$ model) and low competitive ratio (in the $\onela$ model) (also see a related paper by \cite{bianchi12} on switching costs).

Despite this lower bound, the work of Buchbinder et. al.~\cite{buchbinder12} shows an elegant connection between the two frameworks in the specific {\em experts/$\mts$ setting}.  They give a primal-dual based unified algorithmic approach that interpolates between the $\zerola$ problem of learning from experts and the  so-called {\em Metrical Task Systems} ($\mts$) problem: this latter being a $\onela$ problem. Their unified algorithm attains both {\em optimal} regret and competitive ratios, via {\em tuning} certain parameters; previous work in the same direction such as \cite{blum97, blum00} do not obtain optimal competitive ratios. They also obtain near-optimal results for the case of {\em drifting regret}. Buchbinder et al.~\cite{buchbinder12} also note that these interesting results are ``specific to the setting of experts/$\mts$", and ask for more general settings as online convex optimization ($\oco$), ``whether unifying analysis and algorithms exist in such cases as well."

Our current work seeks to address precisely this question.
\subsection{Our Contribution}
Our main focus is on {\em drifting adversaries}. 
We provide an unifying {\em primal-dual} framework to analyze regret (in the $\zerola$ case) for arbitrary norms, and competitive ratio (in the $\onela$ case). We observe that while a primal-dual framework had been introduced by Shalev-Schwarz and Singer~\cite{shalev06}, their framework does not capture the situation of {\em drift} because of the way the constraints for $\opt$ are modeled (a detailed discussion appears in Section~\ref{sec:oco}). Our primal dual framework proposes a methodical approach to formulate new updates in order to minimize drifting regret in arbitrary norms.  Our work is inspired by that of Buchbinder et. al.~\cite{buchbinder12}; however there are some salient differences between \cite{buchbinder12} and the current work. The work of \cite{buchbinder12} considers only the $\ell_1$ norm and thus gives $\lp$-based primal dual algorithms; in order to consider arbitrary norms, it will not suffice to consider $\lp$s.  We consider convex programs and propose primal-dual schema in order to analyze regret. The key contribution of our paper is that we demonstrate how our primal dual schema (and consequently, the $\lp$-based primal-dual schemas of \cite{buchbinder12}) are intimately connected to the Online Mirror Descent ($\omd$) style updates.
In prior work, the papers by \cite{abernethy10} and \cite{blum99} give competitive analysis algorithms using tools from online learning, for instance regularization. However, the role of regularizers in our current work is significantly different from their work.
We interpret the various algorithms provided in Bansal et al.~\cite{bansal10}, Buchbinder et al.~\cite{buchbinder12} in terms of $\omd$ updates; this considerably simplifies and augments the understanding of existing techniques. 

In the $\zerola$ case, we consider the problem of {\em drifting regret}. The key insight that results from our work is the following: in order for an algorithm to perform well wrt drifting regret, it is necessary that the {\em gradient} of the regularizer function be suitably bounded. Thus, in situations where this is not the case a priori, we need to {\em shift} the regularizer, so that this is achieved. 

For the $\onela$ case, we interpret the results of \cite{buchbinder12} in the framework of $\omd$; significantly simplifying the updates. We use the insight gained by this to prove competitive ratios for the case of $\onela$ with the movement cost being the $\ell_2$ distance between successive points chosen. 

Summarising our results:
\begin{itemize}
\item {\em $0$-lookahead:}
We prove drifting regret bounds for arbitrary norms and (slightly) improve existing drifting regret bounds of \cite{buchbinder12} for the $\ell_1$-norm. Since \cite{buchbinder12} reduces the problem of drifting regret to a $\onela$ problem, their bounds for drifting regret assume that the cost functions are {\em non-negative}. We remove this assumption by directly working with drifting regret. In passing, we also provide {\em optimal} guarantees for drifting regret for $\ell_1$-norm, thereby slightly improving the result of \cite{buchbinder12}. For the case of an arbitrary norm $\ell_p$, we prove our results for convex bodies that are $p$-balls. 
\item {\em $1$-lookahead:}
We prove {\em near-optimal} competitive ratios for the $\onela$ problem where the convex body is a $2$-ball, and the movement cost is the $\ell_2$-distance. We provide significant simplifications for the $\onela$ algorithm and primal-dual analysis for the $\mts$ problem considered in \cite{buchbinder12, blum97, blum00}. 
\item {\em Other Results:}
In  the realm of $\zerola$ problems, prior work (see Hazan \& Kale~\cite{hazan10}) has also considered the problem of obtaining regret bounds (against a fixed adversary) that are bounded by some measure of the {\em deviation} of the cost functions (instead of the number of rounds, $T$). Recently, a breakthrough paper by Chiang et al.~\cite{mahdavi12} proved optimal bounds for this problem against a fixed adversary.  Via our techniques, we extend their work  to achieve optimal bounds for the setting of drifting regret.

In the context of $\onela$ problems, we are also able to show optimal bounds when the cost functions are {\em smooth} (\cite{buchbinder12} only consider linear functions). The idea of shifting regularizers enables us to prove competitive ratios for the case where the movement costs are {\em arbitrary} norms, while the convex body is the {\em probability simplex}.

For lack of space, these proofs are deferred to the full version.
\end{itemize}

In ending, the overall conceptual message of our paper is that the problem of drifting adversaries {\em reduces to} the problem of bounding the {\em gradients} of the regularizer in the appropriate norm. 

\subsection{Organization}
We provide the necessary mathematical background and preliminaries in Section~\ref{sec:prelim}. In Section~\ref{sec:oco} we introduce our primal-dual framework and analysis by applying it to the $\oco$ problem for the $\ell_2$-norm. For the $\ell_1$-norm, Bansal et al.~\cite{bansal10} had given an $\lp$-based primal dual approach. We then consider the problem of drifting regret for the $\ell_1$-norm in Section~\ref{sec:drifting_expert}; this subsumes the approach of \cite{bansal10}. In Section~\ref{sec:2-norm1lookahead}, we present our results for the $\ell_2$-norm of the $\onela$ problem.
The simplification of the result of Buchbinder et al.~\cite{buchbinder12} for the $\ell_1$-norm of the $\onela$ problem appears in the Appendix.

\section{Preliminaries}
\label{sec:prelim}
We will let $\cdot$ denote the multiplication of scalars, while $\circ$ will denote the inner product of two vectors.
A
\textbf{norm} of a vector $\mathbf{x}$ is
denoted by $\| \mathbf{x} \|$. The \textbf{dual} \textbf{norm} is
defined as $\| \mathbf{w} \|_{\star} = \sup \{ \mathbf{x} \circ
\mathbf{w}: \| \mathbf{x} \| \leqslant 1 \}$. The most common norms that
will occur in this paper are the $\ell_2$-norm, $\| \mathbf{x} \|_2 =
\sqrt{\mathbf{x} \circ \mathbf{x}}$ and the $\ell_1$-norm $\| \mathbf{x}
\|_1 = \sum_i | x_i |$. 
In general the $p$-norm $\ell_p$ is defined as $\|
\mathbf{x} \|_p = \sqrt[p]{| x_i |^p}$; the dual norm of the
$\ell_p$-norm is the $\ell_q$-norm where $q$ is the \textit{H\"{o}lder
conjugate} of $p$: $1 / q + 1 / p = 1$.
The $\ell_2$-norm is
\textit{self-dual}, while the dual norm of the $\ell_1$-norm is the
$\ell_{\infty}$-norm $\| \mathbf{x} \|_{\infty} = \max_i | x_i |$.
\textbf{H\"{o}lder's Inequality} states that $\|\mathbf{x}\|\cdot \|\mathbf{y}\|_{\star} \geqslant \mathbf{x}\circ \mathbf{y}$ for any norm $\| \cdot \|$. 
The Cauchy-Schwarz inequality is the special case for the $\ell_2$-norm. 
A $p$-ball for $p \in \mathbb{R}^{+}$ with center $k$ and diameter $2D$ is the convex body $\{x: \| x-k\|_p \leq D\}$.

In the setup considered in this paper, the  decision maker (online algorithm) picks an action (a point) every round. This continues for $T$ rounds.\\
In the $\zerola$ setting, 
the decision maker  picks a point $x_{t-1}$ in the convex body $K$ in the $t^{th}$ round, after which it receives a loss function $c_t(x)$ and suffers a {\em penalty} $c_t(x_{t-1})$. The goal of the algorithm is to minimize its {\bf regret} against a {\em fixed policy adversary}: $\min_{u \in K} \{\sum_{t=1}^T [c_t(x_{t-1}) - c_t(u)]\}$.
By fixed policy, we mean that the adversary has to choose a fixed point $u$ against which we evaluate the algorithm's regret. A {\em drifting} adversary is one which is allowed to change its policy to a certain extent.
For the case of drifting adversary, the problem setup uses an extra parameter called the {\em drift} and denoted as $L$. Thus if $u_t$ denote the adversary's policy in the $t^{th}$ round, we require that $\sum_t \|u_t - u_{t-1}\| \leqslant L$ for some $L\geqs 0$. The algorithm has to make predictions so that the {\bf drifting regret} is minimized: $\min \{\sum_{t=1}^T [c_t(x_{t-1}) - c_t(u_{t-1})]\}$ where the minimum is over all {\em drifting} adversaries with drift $L$.

The {\bf Experts Setting} is a special case of $\zerola$ where there are $n$ experts/actions and the decision maker in each round has to choose one of the $n$ actions to play, possibly in a randomized fashion.
Thus this can be thought of as the $\oco$ problem over the probability simplex. 

In the $\onela$ setting, in the $t^{th}$ round, the decision maker perceives the loss function $c_t(x)$ and then picks a point $x_t$. The {\em service cost} paid by the decision maker is $S_1 = \sum_t c_t(x_t)$. In addition, the algorithm also incurs a {\em movement cost} $M = \sum_t \|x_t - x_{t-1}\|$ for a specific norm $\| \cdot \|$. Let $\opt$ denote the optimal offline policy (and its cost) that has access to all the loss functions. The {\em competitive ratio} guaranteed by the decision maker is defined as the minimal $c$ such that for any sequence of cost vectors it holds that $S_1 + M \leqs c\cdot{\opt} + d$ where $d$ is a constant independent of the number of rounds $T$.

Given a function $R$, the \textbf{Bregman Divergence} $B_R ( x, y)$
is defined as:
\begin{eqnarray*}
  B_R ( x, y) & = & R ( x) - R ( y) - \nabla R ( y) \circ ( x - y)
\end{eqnarray*}
Henceforth, we will only consider Bregman divergences for {\em convex} functions $R$:
for such a function, the Bregman divergence is always non-negative: $B_R ( x, y)
\geqslant 0$.

We will use the following properties of Bregman divergences. For a general reference, cf. \cite{bianchi06}.
The following identity is called the $3$-point equality:
\[
  [ \nabla R ( a) - \nabla R ( b)] \circ ( c - b)  =  B_R ( b, a) - B_R ( c, a) + B_R ( c, b)
\]
\eat{
Finally, taking $a=c$ in this:
\[
  [ \nabla R ( a) - \nabla R ( b)] \circ ( a - b)  =  B_R ( b, a) + B_R ( a, b)
\]
Thus we also have (since $B_R(x,y) \geqslant 0$ for any $x,y$) that
}
We also have $[ \nabla R ( a) - \nabla R ( b)] \circ ( a - b) =B_R ( b, a) + B_R ( a, b) \geqslant B_R ( a,b)$.

A function $f$ is called $\sigma$-\textbf{strongly convex} over a domain $D$ wrt a norm $\|
\cdot \|$ if for all $x, y\in D$ it holds that $B_f ( x, y) \geqslant 
\frac{\sigma\| x - y \|^2}{2} $. Thus, for the $\ell_2$-norm and the domain $D = \mathbb{R}^n$, the function $f(x) = \frac{\|x\|^2_2}{2}$ is $1$-strongly convex (over the entire domain).  For the $\ell_1$-norm, and the domain $D = \{x : x > 0 ~\&~ \|x\|_1 \leqslant B\}$,  the {\em negative entropy} function $f(x) = \sum_i x_i{\log x_i}$ is $\frac{1}{B}$-strongly convex (cf. the excellent survey~\cite{shalev-survey12}, page 32, Example 2.5).

Given a convex body $K$,  and a strongly convex regularizer $R$, the \textbf{Bregman projection} of a point $y$ onto
$K$ wrt $R$ is defined as the point $x=\Pi_K(y)= \mathrm{argmin}_{u \in K} B_R ( u, y)$ .
When the convex body $K$ is clear from the context, we will drop the subscript for $\Pi$. 

The {\bf Fenchel conjugate} of a function $f  (see \cite{shalev-survey12}, Table 2.1):
S \rightarrow \mathbb{R}$ is defined as $f^{\star} ( \theta) = \sup_w \{ w
\circ \theta - f ( w) \}$. If $f$ is closed and convex, then the Fenchel
conjugate of $f^{\star}$ is $f$ itself. For instance, 
the {\em norm} function $f(x) = \|x\|$, has the following Fenchel dual (see Example 3.26 of Boyd Vanderberghe \cite{BV04}):
\begin{eqnarray*}
f^{\star}(y)&=& 0 \mbox{,    if $\|y\|_{\star}$ $\leqslant$ $1$} \\
&=& \infty \mbox{,   otherwise.}
\end{eqnarray*}

In this section, we will provide a general form of the lemma that we will use in the following sections. Note that a version of this lemma is already present in various earlier works in the literature, for example~\cite{mahdavi12} (Prop. $7$). However, we could not find this form of the lemma anywhere and so, 
for completeness we include a proof in the Appendix (Section~\ref{app:proj}).

\begin{lemma}\label{lem:proj}
  (\textbf{Projection Lemma) }
Let $K\subseteq \mathbb{R}^n$ be a convex body and let $\mathbb{R}^n$ be equipped with a norm  $\| \cdot \|$.  Let $R :D \rightarrow \mathbb{R}$ be a regularizer function (over domain $D$ such that $K\subseteq D$) that is $\sigma$-strongly convex with
  respect to the norm $\| \cdot \|$.
Let $y_1$ and $y_2$ be two points
(possibly outside $K$),
 and let $x_1$ and $x_2$ be the Bregman
  projections of $y_1$ and $y_2$ respectively, on to the convex body $K$. Then
  the following statement holds:
  \[ 
\sigma{\cdot}\| x_1 - x_2 \| \leqslant \| \nabla R ( y_1) - \nabla R ( y_2) \|_{\star}
  \]
  where $\| \cdot \|_{\star}$ is the \textbf{dual} norm of $\| \cdot \|$.
\end{lemma}

The {\bf Online Mirror Descent} (hereafter, shortened as $\omd$) update was first presented by Beck \& Teboulle~\cite{beck03}. Given a strongly convex regularizer $R$, the update is as follows:
In the $t^{th}$ iteration:
\begin{eqnarray*}
\nabla R(y_t) & =& \nabla R(x_{t-1}) - \eta c_t\\
x_t & = & \Pi_K(y_t)
\end{eqnarray*}

In this paper, we will assume that the cost functions are {\em linear}. This is without loss of generality; it is well known that for regret settings, one can reduce the arbitrary convex case to the case of linear functions.
We will let $D$ denote the {\em diameter} of the relevant convex body $K$, and $G$ denote the supremum of the {\em gradients} of loss functions, whenever applicable. 
\section{The $0$-lookahead setting}
\label{zero}

\subsection{Warmup: A new Fenchel dual for OCO}
\label{sec:oco}

In this section, we will describe a new Fenchel dual for the problem of Online Convex Optimization (OCO). This program is different from the Fenchel dual presented by Shalev-Shwartz \& Singer~\cite{shalev06}. The essential difference between this program and the one in \cite{shalev06} lies in the modeling of the offline optimum, $\opt$. The design we adopt leads to a couple of advantages for the convex programs we describe. In our framework, we are directly able to capture the problem of {\em constrained} $\oco$. Here, the dual objective is directly related to the {\em cost} of projection to the contrained body. To the best of our knowledge, the Fenchel duals presented in \cite{shalev06} do not have this feature. Additionally, given the technique as to how we model $\opt$ as a convex program, our framework and corresponding results are easily extendable to the situation of {\em drifting regret} (see Section~\ref{sec:drift}). 


The basic theme of the primal dual schema is that we model $\opt$ as a  (primal) convex program; and we devise an algorithm that is then compared to the {\em dual} objective function of the primal program. Note that the dual objective is a {\em lower bound} for the primal program, i.e. $\opt$. Here, the objective of the algorithm is $\sum_{t=1}^{T} c_t\circ x_{t-1}$ whereas the objective of the offline $\opt$ is $\min_{x\in K}\{\sum_{t=1}^{T} c_t\circ x\}$. In order to transform the computation of $\opt$ into a convex program, \cite{shalev06} uses the constraints $x_t = x_0$ (for all $t=1,\cdots, T$). In the following, we consider penalties of $\infty$ attached to a deviation of $x_t$ from $x_{t-1}$ in a specific norm $\|\cdot\|$. Thus, this is inspired by the work of Bansal et al~\cite{bansal10}, where they consider such a program for the $\ell_1$-norm, $\|\cdot\|_1$. Since \cite{bansal10} consider only the $\ell_1$-norm, they are able to derive an $\lp$ for $\opt$. However, for our purposes, we will have to involve a convex program to model $\opt$. 

It is noteworthy, that for the $\ell_1$-norm case, since the convex body is a probability simplex, the feasible points are such that {\em every} component is {\em non-negative}.  Thus the Fenchel duals to be considered are {\em constrained} by $x_{i,t} \geqslant 0$. This leads to a different convex program (and, since it is the $\ell_1$-norm, a {\em linear} program) than for the case of $\ell_p$ (or $2 \geqs p > 1$) norms.

\noindent {\bf The $p$-norm case:}
The convex program modelling $\opt$ is given in $\cp_1$.
\eat{
\begin{eqnarray*}
  \min &  & \sum^T_{t = 1} c_t \circ x_{t - 1} + \infty \cdot \sum_t \|z_t\|_p \\
  s.t. & \forall t \geqslant 0 : & \| x_t \|_p \leqslant D\\
  & \forall i, t \geqslant 1 : & z_{i, t} = x_{i, t} - x_{i, t - 1}
\end{eqnarray*}
}

We may construct the Lagrange dual of $\cp_1$ by lifting the constraints to the objective function; we will keep the Lagrange multiplier $a_t$ for the constraint $( \| x_t \|_p - D)\leqslant 0$ and the variable $b_t$ for the 
constraint $z_t = x_{t} - x_{t - 1}$. Thus, while $a_t$ is a scalar, $b_t$ is a vector. The Lagrange dual expression ${\cal L}$ (as a function of the dual variables $a_t, b_t$) then is 
\begin{eqnarray*}
  {\cal L} (a, b)= &\min  & \sum^T_{t = 1} [c_t \circ x_{t - 1} + \infty \cdot \sum_t \|z_t\|_p ]+\sum^T_{t = 0} a_t\cdot (\| x_t \|_p - D) + \sum^T_{t = 1} b_t\circ (x_{i, t} - x_{i, t - 1} - z_{i, t})
\end{eqnarray*}
which separates out in terms of the variables $x_t$ and $z_t$ reads:
\begin{eqnarray*}
  {\cal L} (a, b)= &\min  & \sum^T_{t = 1} [x_t \circ (c_{t+1} + b_t - b_{t+1}) + a_t \cdot \|x_t\|_p] + \sum^T_{t=1} [\infty \cdot  \|z_t\|_p - b_t\circ z_t] - D\sum^T_{t = 0} a_t 
\end{eqnarray*}
Given the Fenchel duals of the norm function, cf. Section~\ref{sec:prelim}, let us write the dual of the program $\cp_1$ as $\cp_2$. Here $q$ is such that $\ell_q$ is the dual norm of $\ell_p$ (thus, $q = p/(p-1)$).

\[
\begin{array}{cll}
\cp_1: &\\[0.25em]
\min &{\displaystyle  \sum^T_{t = 1} c_t \circ x_{t - 1} }&{\displaystyle + \infty \cdot \sum_t \|z_t\|_p}\\[1.0em]
{\rm s.t.} &{\displaystyle 
\| x_t \|_p \leqslant D}&{\forall} {T-1} \geqs t \geqslant 0 \\[0.25em]
&{\displaystyle z_{t} = x_{t} - x_{t - 1}}&{\forall} i,t \geqslant 1
\end{array}
\quad
\begin{array}{||cll}
\cp_2: &\\[0.25em]
\max &{\displaystyle - D\sum^{T-1}_{t = 0} a_t }\\[1.0em]
{\rm s.t.} &{\displaystyle 
 a_t  \geqslant \|b_{t + 1} - b_{t} - c_{t + 1}\|_q}& {\forall} t\geqslant 1 \\[0.25em]
&{\displaystyle a_0 
  \geqslant \|b_1 - c_1\|_q}&t=0\\[0.25em]
&{\displaystyle \|b_t\|_q \leqslant \infty} & {\forall} t\geqslant 1\ 
\end{array}
\]

The goal of this section is to prove the following theorem:
\begin{theorem}
\label{thm:oco}
Consider the convex programs $\cp_1$ and $\cp_2$. 
Let $c_1, c_2, \cdots, c_T$ be a sequence of {\em linear} cost functions, and let $R$ be a strongly convex regularizer for the $\ell_p$-norm. Let the norm dual to the $\ell_p$-norm be the $\ell_q$-norm where $q = p/(p-1)$. Then there exists a dual incrementing algorithm that generates a sequence of primal vectors $x_0, x_1, \cdots, x_{T-1}$ along with a sequence of feasible dual variables $a_t$ ($t = 0, \cdots, T-1$) and $b_t$ ($t = 1, \cdots, T$) such that the following holds: \\
\begin{eqnarray*}
\sum_{t=0}^{T-1} c_t \circ x_{t-1} &\leqslant &\frac{[R(x_{T}) - R(x_0)]}{\eta} + \frac{[x_0\circ \nabla R(x_0) - x_{T}\circ \nabla R(x_T)]}{\eta} +\eta\|c_t\|^2_q - D\sum_t a_{t-1}
\end{eqnarray*}
\end{theorem}

\begin{proof}(Sketch)
We will provide the calculations for the case $p = 2$. For an arbitrary $p$, one has to take the corresponding regularizer into account and the calculations are similar.

We will assume that the updates are made according to the OMD paradigm (see Section~\ref{sec:prelim}). 
Given this, set the dual variables as follows:
\begin{eqnarray*}
{\eta}a_0 & = & \| \nabla R(y_1) - \nabla R(x_1) - \nabla R(x_0) \|_2 \\
t \geqslant 1: {\eta}a_t &=& \|\nabla R(y_{t+1}) - \nabla R(x_{t+1})\|_2  \\
t \geqslant 1:{\eta}b_t & = & -\nabla R(x_t)
\end{eqnarray*}
Thus, in the $t^{th}$ iteration, we receive the cost function $c_t$ and also update the dual variables $a_{t-1}$ and $b_t$. 
Let us check the feasibility of the dual constraints. In fact, these dual updates will maintain the constraints $a_t  \geqslant \|b_{t + 1} - b_{t} - c_{t + 1}\|_2$ with {\em equality}:
\begin{eqnarray*}
\|b_{t + 1} - b_{t} - c_{t + 1}\|_2 & = &( \| -\nabla R(x_{t+1}) + \nabla R(x_t) + \nabla R(y_{t+1}) - \nabla R(x_t) \|_2)/\eta\\
& = & (\| \nabla R(y_{t+1}) -\nabla R(x_{t+1}) \|_2)/\eta \\
& = & a_t
\end{eqnarray*}

Interestingly only the above description of $a_t$ does not seem to suffice to prove an upper bound on the primal cost in terms of the dual objective function. This necessitates doing something essentially different from the analysis given in \cite{buchbinder12} (for the case of the $\ell_1$-norm). Given our updates, in addition to the tightness of the constraint above, we will also {\em implicitly} maintain the following identity: 
\begin{eqnarray*}
a_t\cdot \frac{x_{i,t+1}}{\sqrt{\sum_i x_{i,t+1}^2}} & = & \frac{[\nabla R(y_{t+1}) -\nabla R(x_{t+1})]_i}{\eta}
\end{eqnarray*}
Since we project the point $y_t$ to the point $x_t$ lying on the boundary of the convex body, we have the following {\em complementary slackness}~\footnote{Note that this is not exactly {\em complementary slackness} since $a_t$ corresponds to the constraint $\|x_t\|_2 \leqs D$.}
condition: $a_t > 0 \Rightarrow \|x_{t+1}\|_2 = D$.

However we need to prove that it is indeed possible to have all the $x_{i,t}$'s simultaneously satisfy the above equality, as well as this complementary slackness condition. 
In order to bound the service cost, we will appeal to the following Lemma~\ref{lem:oco_base} with $u_0 = u_1=\cdots = u_{T-1}$. 

\begin{lemma}
\label{lem:oco_base}
Let $R$ be a strongly convex regularizer w.r.t a norm $\|\cdot \|$ over a convex body $K$. Let $c_1,c_2,\ldots,c_T$ be the sequence of linear cost functions that adversary presents. Also let $x_0,x_1,\ldots,x_T$ be the sequence of primal vectors our algorithm predicts according to the update rule {reference} and $u_0,u_1,\ldots,u_{T-1}$ be any $T$ vectors in $K$. Denote $A=B_{R}{(u_0,x_0)}-B_{R}{(u_{T-1},x_T)}+\nabla R(x_0)\circ u_0 - \nabla R(x_T)\circ u_{T-1}$ , $B=\sum_{t=1}^{T}{B_R(x_{t-1},x_t)}$ and $C=\sum^T_{t = 1} [ \nabla R ( x_t) - \nabla R ( y_t)] \circ x_{t - 1}$. Then
\begin{eqnarray}
  \eta S_0 = \eta \nosymbol \sum_{t = 1}^T c_t \circ x_{t - 1} & = & \sum^{T-1}_{t=1}(R( u_{t}) - R ( u_{t-1})) + A + B + C \\
B+C &\leqs& \eta^2\|c_t\|_{\star}^2 - {\eta}D\sum_{t=1}^{T}a_{t-1}
\end{eqnarray}
\end{lemma}


This completes the proof sketch of the main theorem of this section. The proof of the complementary slackness condition and that of Lemma~\ref{lem:oco_base} are presented in the Appendix, 
\end{proof}
\QED

\subsection{Drifting Regret for Expert Settings}
\label{sec:drifting_expert}
We now proceed to consider the problem of expert systems. Here, the underlying convex body $K$ will be the probability simplex.
We will study regret against {\em drifting adversaries} with drift parameter $2L$:
the regret is measured against the optimal sequence of vectors $u^{\star}_0, u^{\star}_1,\ldots, u^{\star}_T$ such that $\frac{1}{2}\sum_{t = 1}^T \| u^{\star}_t - u^{\star}_{t - 1} \|_1 \leqslant L$. The factor of $2$ is included so as to simplify the $\lp$ formulation of the problem: this standard transformation (see~\cite{buchbinder12}) corresponds to tracking just the {\em increments} in the components of $(u^{\star}_t - u^{\star}_{t - 1})$.  The primal program modelling $\opt$ is presented as $\lp_3$, and its dual is presented as $\lp_4$.

\[
\begin{array}{cll}
\lp_3: &\\[0.25em]
\min &{\displaystyle  \sum^T_{t = 1} c_t \circ x_{t - 1}}\\[1.0em]
{\rm s.t.} &{\displaystyle 
\sum_i x_{i,t} = 1}&{\forall} t \geqslant 0 \\[0.25em]
&{\displaystyle z_{i,t} \geqs x_{i,t} - x_{i,t - 1}}&{\forall} i,t \geqslant 1\\[0.25em]
&{\displaystyle \sum_{t=1}^{T-1} \sum_i z_{i,t} \leqslant L}&\\
&{\displaystyle z_{i,t}, x_{i,t} \geqslant 0} &{\forall} i,t \geqslant 0 \
\end{array}
\quad
\begin{array}{||cll}
\lp_4: &\\[0.25em]
\max &{\displaystyle - \sum^{T-1}_{t = 0} a_t - \alpha{\cdot}L   }\\[1.0em]
{\rm s.t.} &{\displaystyle 
 a_t  \geqslant b_{i,t + 1} - b_{i,t} - c_{i,t + 1}}& {\forall} i, t\geqslant 1 \\[0.25em]
&{\displaystyle a_0 
  \geqslant b_{i,1} - c_{i,1}}&t=0,{\forall} i\\[0.25em]
&{\displaystyle 0 \leqs b_{i,t} \leqslant \alpha} & {\forall}i, t\geqslant 1\\[0.25em]
&{\displaystyle \alpha \geqslant 0}&{\forall} t \geqslant 0\\[0.25em]
$$\
\end{array}
\]
This $\lp$ formulation is similar to the one used by \cite{bansal10,buchbinder12}. Note that $\alpha$ is a dual variable corresponding to the drifting constraint.
\begin{theorem}
\label{thm:drift1norm}
Consider the convex programs $\lp_3$ and $\lp_4$. 
Let $c_1, c_2, \ldots, c_T$ be a sequence of {\em linear} cost functions and $R$ is a strongly convex regularizer w.r.t $1$-norm in the probability simplex. Then there exists a dual incrementing algorithm that generates a sequence of primal vectors $x_0, x_1, \ldots, x_{T-1}$ along with a sequence of feasible dual variables $a_t$ ($t = 0, \ldots, T-1$), $b_{i,t}$ ($i = 1, \ldots, n$,$t = 1, \ldots, T$) and $\alpha$ such that the following holds: \\
\begin{eqnarray*}
\sum_{t=0}^{T-1} c_t \circ x_{t-1} &\leqslant & \frac{1}{\eta}\left( 3 ( L + 2) {\ln n}  \right) + \eta \|c_t\|^2_{\infty} + \left( - \sum_t a_{t-1} - \alpha L \right)
\end{eqnarray*}
\end{theorem}

\begin{proof}(Sketch)
We will use the updates in the style of $\omd$ (as also in Section~\ref{sec:oco}). The vanilla regularizer used for the $\ell_1$-norm is the {\em negative entropy} function $\sum_i x_i {\ln x_i}$. However, the definition of the regularizer $R$ will be changed in the following, which is essential for maintaining the dual feasibility. We set the dual variables $b_{i,t}$ and $\alpha$ as follows:
\begin{eqnarray}
\forall t \geqs 1 :& \eta b_{i,t} & =  [\nabla R ( \tmmathbf{1}) -  \nabla R ( x_t)]_{i} \nonumber\\
  & \eta\alpha & = {\ln n}\nonumber
\end{eqnarray}
The constraint $b_{i,t} \leqs \alpha$ may be violated if we were to define $b_{i,t}$ according to the vanilla regularizer, since some $x_{i,t}$ may be $0$ and then corresponding $b_{i,t}$ may become unbounded. The cardinal step is to thereby {\em shift} the regularizer by an amount $\theta$ so that this does not happen. Thus we will consider the regularizer $R(x_t) = \sum_i (x_{i,t} + \theta){\log (x_{i,t} + \theta)}$. In order for $b_{i,t}$ to be bounded by $\alpha$, calculations yield that $\theta = 1/(e^{\eta \alpha} - 1)$. 

In order to set $a_t$, we will need a fact about Bregman projections; this is encapsulated as Claim~\ref{claim:breg} proven in Appendix (Section~\ref{app:bregman}). This claim may be of independent interest.
\begin{claim}
\label{claim:breg}
Consider the Bregman projection  $x^{*}$ of the point $y^{*}$, wrt a regularizer $R$, onto the body $K$. 
There exists a constant $\lambda$ independent of $i$ such that the following hold:
\begin{eqnarray*}
\forall{i}: x^{*}_i > 0 \Rightarrow [\nabla R(y^{*}) - \nabla R(x^{*})]_i &=& \lambda \\
\forall{i}:  [\nabla R(y^{*}) - \nabla R(x^{*})]_i &\leqs& \lambda
\end{eqnarray*}
\end{claim}
We set $a_t$ to be the $\lambda$  guaranteed by Claim~\ref{claim:breg}, when applied on $y^{*} = y_{t+1}$. It is easy to show that the dual thus constructed is feasible.

To complete the proof of the Theorem~\ref{thm:drift1norm} we invoke Theorem~\ref{thm:drift2norm}(Appendix~\ref{subsec:drift_p_norm}). The final calculations are completed in the Appendix~\ref{subsec:drift_expert}.
\QED
\end{proof}
\section{The $1$-lookahead setting}
\label{one.lookahead}
In this section, we study the $\onela$ problem concerning a convex body $K$. Recall that in iteration $t$, the  algorithm is allowed to choose the point $x_t \in K$ {\em after} it receives the cost function $c_t$. However, it has to pay the cost of {\em moving} between the points $x_{t-1}$ and $x_t$. Thus, the overall cost includes the {\em service cost}  $S_1 = \sum_t c_t\circ x_t$ and the {\em movement cost} $M = \sum_t \|x_t - x_{t-1}\|$, for some specific {\em movement norm} $\|\cdot \|$. 
Given this scenario, we are concerned with designing algorithms with reasonable {\em competitive ratios} for this online problem ($\onela$). Note that the work of \cite{buchbinder12} considers the $\onela$ problem for the convex body being the {\em probability simplex} and the {\em movement norm} being the $\ell_1$-norm. In the following we consider the $\onela$ problem with the movement norm being the $\ell_2$-norm, and the convex body being a $2$-ball. 
\subsection{The $\ell_2$-norm}
\label{sec:2-norm1lookahead}
Let $2D$ be the diameter of the $2$-ball where the chosen points are constrained to lie. 
The main result of this section is an $O(D)$-competitive algorithm for the $\onela$ problem in the $\ell_2$-norm.

As discussed in Section~\ref{sec:intro}, for competitive analysis to be meaningful, the chosen points have to avoid {\em negative coordinates}. Motivated by this, let $k$ denote the {\em center} of the $2$-ball and let $k_{\min}$ denote the {\em minimum} coordinate of $k$. Then we must have that $k_{\min} \geqs D$. For similar reasons, every component of $c_t$ has to be {\em non-negative}. Let $\opt$ denote the optimum offline algorithm for a specific sequence of cost functions $c_1, \ldots, c_T$. 
We prove the following:
\begin{theorem}
\label{thm:1lookahead2norm}
Consider the $\onela$ problem for the $2$-ball of diameter $2D$, centered at $k$ with $k_{\min} \geqs (D+\epsilon)$ (for some $\epsilon > 0$), and with movement cost in the $\ell_2$-norm. Then there exists a dual incrementing algorithm (taking as extra input a parameter $\eta$) that achieves the following: the service cost $S_1$ satisfies: $S_1 \leqs \opt + \frac{D}{\eta}$ and the movement cost $M$ satisfies: $M \leqs \frac{\eta}{\eps}\opt$.
\end{theorem}
\begin{proof}
For lack of space, we present the primal convex program and its Fenchel dual in section~\ref{subsec:2BallDualFeas} of the Appendix. The dual objective function is ${\mathcal D} = \sum_{t=1}^T c_t \circ k - D\sum_{t=0}^T a_t$ where $a_t$ is the dual variable corresponding to the constraint that $x_t \in K$. We use the regularizer $R(x) = \frac{\| x - k \|^2}{2}$ and the update is as follows:
\begin{eqnarray}
  \nabla R (y_t-k) & = & \nabla R (x_{t - 1}-k) - \eta c_t \nonumber\\
  x_t-k & = & \Pi (y_t-k) \nonumber
\end{eqnarray}
\begin{eqnarray}
t \geqslant 1: {\eta}a_t &=& \|\nabla R(y_t-k) - \nabla R(x_t-k) \| \nonumber \\
{\eta}a_0 &=& \| \nabla R(x_0-k) \| \nonumber
\end{eqnarray}

Like in earlier sections, we can claim the feasibility of these updates (cf. Section~\ref{subsec:2BallDualFeas} in Appendix). The service cost may be analysed in a manner similar to that of Sections~\ref{sec:oco},\ref{sec:drifting_expert} and is deferred to  Section~\ref{sec:1LA2normservice} in the Appendix. The more interesting component of the objective function is the movement cost and we bound it as follows. Let $M_t$ denote the movement cost accrued in the $t^{th}$ iteration, and let 
$\Delta {\mathcal D}_t$
denote the change in the dual objective in iteration $t$. Claim~\ref{claim:pythagorean} proved in Section~\ref{sec:1LA2normservice} of Appendix gives us $a_t \leqs \|c_t\|_2$. By Lemma~\ref{lem:proj} we have that $M_t  =  \| x_{t} - x_{t - 1} \|\leqslant  \| \nabla R (y_t) - \nabla R (x_{t - 1}) \| =  \eta \| c_t \|_2$. Note that for an arbitrary vector $v$, we have that $\|v\|_2 \leqs \|v\|_1$. 
It follows that $M_t \leqs \eta \| c_t \|_1$, and $a_t \leqs \| c_t \|_1$. 

Since {\em every} component of $c_t$ is {\em non-negative}, $\|c_t\|_1 \leqs \tmmathbf{1}\circ c_t$. Thus, $k\circ c_t \geqs (D+ \eps) \tmmathbf{1}\circ c_t \geqs (D+\eps)\|c_t\|_1 \geqs \eps \|c_t\|_1 + D\cdot a_t$.

 Transposing terms, we have that $\Delta {\mathcal D}_t = k \circ c_t - D\cdot a_t \geqs \eps \|c_t\|_1$. Therefore, we have that $M_t \leqs \eta \|c_t\|_1 \leqs \frac{\eta}{\eps}{\Delta {\mathcal D}_t}$; summing over all $t$ gives the required bound for the movement cost.

This concludes the proof of the theorem.
\end{proof}
\QED

In the statement of Theorem~\ref{thm:1lookahead2norm}, via setting the value of $\eta = D$  we obtain a competitive ratio of $D/\eps$ for the $\onela$ problem for the $\ell_2$-norm where the $2$-ball is centered at $k$ with $k_{\min} \geqs (D+\eps)$.




\newpage


\bibliography{main}
\newpage
\noindent 
{\large \bf Appendix:}

\section{Projection Lemma}
\label{app:proj}
\begin{proof}[of Lemma~\ref{lem:proj}]
 Since $x_1$ is the Bregman projection of $y_1$ on 
  the convex body $K$, we have that the gradient of the Bregman divergence at
  the point $x_1$ cannot have any direction of decrease that leads inside the
  convex body $K$, in particular in the direction $x_2 - x_1$ (noting that $x_2\in K$). Thus,
  \begin{eqnarray*}
    {}[ \nabla R ( x_1) - \nabla R ( y_1)] \circ ( x_2 - x_1) & \geqslant & 0
  \end{eqnarray*}
  Analogously, considering the points $x_2$ and $y_2$ and letting $x_1$ be
  the ``other'' point in $K$, we have:
  \begin{eqnarray*}
    {}[ \nabla R ( x_2) - \nabla R ( y_2)] \circ ( x_1 - x_2) & \geqslant & 0
  \end{eqnarray*}
  Adding the two inequalities and a little manipulation gives us:
  \begin{eqnarray*}
    {}[ \nabla R ( y_1) - \nabla R ( y_2)] \circ ( x_1 - x_2) & \geqslant & [
    \nabla R ( x_1) - \nabla R ( x_2)] \circ ( x_1 - x_2)\\
    & = & B_R ( x_1, x_2) + B_R ( x_2, x_1) \\
    & \geqslant & \sigma{\cdot}\| x_1 - x_2 \|^2
  \end{eqnarray*}
  where the last inequality follows from the fact that $R$ is a $\sigma$-strongly convex regularizer. A final application of H\"{o}lder's inequality gives us the result. 

  Also note that this gives us:
  \begin{center}
      $\frac{1}{\sigma} \| \nabla R ( y_1) - \nabla R ( y_2) \|^2_{\star} \geqslant [ \nabla R (
    y_1) - \nabla R ( y_2)] \circ ( x_1 - x_2) \geqslant B_R ( x_1, x_2)$
  \end{center}
\end{proof}

\section{Bregman Projections and Dual Feasibility}
\label{sec:dualFeas}	
In this section we make explicit the connection between the dual updates we make in the various sections and the Bregman projections along the regularizer $R$ which gives the prediction points at every iteration. This allows us to show dual feasibility of our updates when the convex body under consideration are either the probability simplex ($\ell_1$-norm) or the $2$-ball centered at $k$ ($\ell_2$-norm)

\subsection{Probability Simplex}
\label{app:bregman}
In this section, we will consider the {\em shifted} regularizer $R(x) = \sum_i (x_i + \theta) \ln (x_i + \theta)$  for a specific $\theta$. For the body of this paper, we have considered the value of $\theta = \frac{1}{e^{\eta\alpha}- 1}$; however this is not relevant for the current discussion. Let us consider the Bregman projection along the regularizer $R$ on to the convex body, the probability simplex $K = \{ x: \|x\|_1 = 1, x_i \geqs 0\}$. 
Given a point $y^{*}$, let $x^{*}$ be its Bregman projection along the regularizer $R$:
\[
x^{*} = \operatorname*{arg\,min}_{x\in K} B_R(x, y^{*})
\]

We want to prove the following 
\begin{claim}
\label{claim:bregman}
Consider the Bregman projection $x^{*}$ of the point $y^{*}$ onto the body $K$. 
There exists a constant $\lambda$ independent of $i$ such that the following hold:
\begin{eqnarray*}
\forall{i}: x^{*}_i > 0 \Rightarrow [\nabla R(y^{*}) - \nabla R(x^{*})]_i &=& \lambda \\
\forall{i}:  [\nabla R(y^{*}) - \nabla R(x^{*})]_i &\leqs& \lambda
\end{eqnarray*}
\end{claim}

\begin{proof}
We will form the Lagrangian of the constrained Bregman projection problem:
\[
x^{*} = \operatorname*{arg\,min}_{x\in K} B_R(x, y^{*})
\]
Thus, we will have dual variables $\kappa_i \geqs 0$ for each constraint $x_i \geqs 0$, and 
$\lambda$ for the constraint $\sum_i x_i = 1$. The Lagrangian expression (after expanding out the definition of the Bregman divergence) is:
\[
\cL(\kappa_i, \tau) = R(x) - R(y^{*}) - \nabla R(y^{*})\circ (x- y^{*}) - \kappa_i{x_i} - \lambda{(1 - \sum_i x_i)}
\]
Thus the Lagrangian dual problem is 
\[
\max_{\kappa\geqs \tmmathbf{0}, \lambda} \min_{x} \cL(\kappa, \lambda)
\]
For the inner problem, by differentiating wrt $x$, at the optimum value $x^{*}$, we get that there are
optimal Lagrangian dual values $\kappa_i \geqs 0$ and $\lambda$ such that:
\[
[\nabla R(x^{*}) - \nabla R(y^{*})]_i - \kappa_i + \lambda = 0 
\]
Rearranging we have: 
\[
[\nabla R(y^{*}) - \nabla R(x^{*})]_i  = \lambda - \kappa_i \nonumber 
\]
Since $\kappa_i \geqs 0$ for all $i$, we have:
\[
[\nabla R(y^{*}) - \nabla R(x^{*})]_i  \leqs \lambda
\]
Also note that when complementary slackness holds, it is true that if $x_i > 0$, then $\kappa_i = 0$. 
So, under that condition, we have that if $x_i > 0$, then $[\nabla R(y^{*}) - \nabla R(x^{*})]_i  = \lambda$.

In the specific setting that we are considering, Slater's condition holds and hence strong duality of the Lagrangian and complementary slackness holds. For details on Slater's condition and strong duality, refer to \cite{BV04}.

\end{proof}

\QED

\subsection{$2$-ball with center $k$}
\label{subsec:2BallDualFeas}
We set out by writing the primal and Fenchel dual convex programs for both 0-lookahead and 1-lookahead settings where the convex body under consideration is the the $2$-ball of diameter $2D$ centered at $k$. Define $\|\cdot\|$ to be $\ell_2$-norm.
\[
  \begin{array}{cll}
\cp_5:&\\[0.25em]
  \min & {\displaystyle\sum^T_{t = 1} c_t \circ x_t + \sum_{t =
  1}^T \| z_t \| 
}\\[1.0em]
  {\rm s.t} & \| x_t - k \|_2 \leq D & \forall t = 0 \ldots T \\[0.25em]
  & z_t = x_t - x_{t - 1} & \forall t = 1 \ldots T  \\
  \end{array}
\quad
  \begin{array}{||cll}
\cp_6:&\\[0.25em]
  \max & {\displaystyle\sum_{t = 1}^T c_t \circ k - D\cdot\sum_{t = 0}^T a_t}\\[1em]
  \rm{s.t} & \| b_t \|_2 \leq 1 & \forall t = 1 \ldots T \\[0.25em]
  & a_t \geq \| b_{t + 1} - b_t - c_t \|_2 & \forall t = 1 \ldots T  \\[0.25em]
  & a_0 \geq \| b_1 \|_2 \\
\end{array}
\]
Recall that the mirror function used for $2$-ball is $R (x) = \frac{\| x - k \|^2}{2} \nosymbol$. At
the $t$th iteration, we receive a cost vector $c_t$ and then predict $x_t$ using the updates:
\begin{eqnarray}
  \nabla R (y_t-k) &=& \nabla R (x_{t - 1}-k) - \eta c_t \nonumber\\
  x_t-k & = & \Pi (y_t-k) \nonumber
\end{eqnarray}
The corresponding dual updates are:
\begin{eqnarray*}
t \geqslant 0: {\eta}b_{t + 1} &=& -\nabla R (x_t - k)\\
t \geqslant 1: {\eta}a_t &=& \|\nabla R(y_t-k) - \nabla R(x_t-k) \| \\
{\eta}a_0 &=& \| \nabla R(x_0-k) \|
\end{eqnarray*}
Let us check for feasibility of the dual constraints. The first set of constraints hold due to the fact that $\|x_t - k\|\leqslant D$ and the fact that eventually the parameter $\eta$ would be set to $D$ to obtain the claimed competitive ratio. The second inequality is indeed satisfied with equality by the choice of $a_t$.
\begin{claim}
At a particular iteration $t$, let $a_t$ be the above dual update. If $a_t > 0$, then
\begin{equation}
\label{eqn:2_norm_identity}
\forall i, \eta a_t\cdot \frac{[x_t - k]_i}{\|x_t-k\|}  =  [\nabla R(y_{t}-k) -\nabla R(x_{t}-k)]_i
\end{equation}
\end{claim}
\begin{proof}
The proof for existence of the identities uses the fact that the Bregman Projection of a point $y_{t}-k$ along $R$ onto the $2$-ball is given by 
\begin{equation}
\label{eqn:2_norm_proj}
x_{t}-k = D\cdot\frac{y_{t}-k}{\|y_{t}-k\|} 
\end{equation}
The RHS of (\ref{eqn:2_norm_identity}) gives
\begin{eqnarray*}
\frac{[y_t - k]_i - [x_t - k]_i}{\eta} & = & \frac{1}{\eta}\cdot[y_t - k]_i \left( 1 -\frac{D}{\| [y_{t} - k \|_2} \right)\\
& = & \frac{1}{\eta \cdot \| y_{t + 1} - k \|_2}[y_t - k]_i (\| y_{t} - k \|_2 - D)
\end{eqnarray*}
At this point, we make a crucial observation which effectively follows from the complementary slackness conditions of the Bregman projection convex program. When $a_t > 0$, it indicates that $y_t$ has been projected back to $x_t$ inside the $2$-ball. This implies $\| x_{t} - k \|_2 = D$ by~(\ref{eqn:2_norm_proj}). So 
\begin{eqnarray*}
a_t . \frac{[x_{t} - k]_i}{\| x_{t} - k \|_2}& = & \frac{a_t\cdot [x_t - k]_i}{D} \\
& = & \frac{1}{\| y_{t + 1} - k \|_2} .a_t . [y_t - k]_i
\end{eqnarray*}
So if $a_t = \frac{1}{\eta} (\| y_{t + 1} - k \|_2 - D)$, or $a_t = \frac{1}{\eta} (\| y_{t + 1} - k \|_2 - \| x_{t + 1} - k \|_2)$ then identity is satisfied for all $i$.
\QED \\
\end{proof}

\section{$0$-lookahead : OCO}
\label{sec:oco_appendix}
\begin{proof}[of Lemma~\ref{lem:oco_base}]

\textbf{Part 1:}
\begin{eqnarray}
\eta \nosymbol \sum_{t = 1}^T c_t \circ x_{t - 1} & = & 
\sum^T_{t
  = 1} [ \nabla R ( x_{t - 1}) - \nabla R ( y_t)] \circ x_{t - 1}
\nonumber\\
  & = & \sum^T_{t = 1} [ \nabla R ( x_{t - 1}) - \nabla R ( x_t)] \circ x_{t
  - 1} + \sum^T_{t = 1} [ \nabla R ( x_t) - \nabla R ( y_t)] \circ x_{t - 1}
  \nonumber\\
   & = & 
\sum^T_{t = 1} [ \nabla R ( x_{t - 1}) - \nabla R (
  x_t)] \circ ( x_{t - 1} - u_{t-1})
 + \nonumber\\
  &  & 
\sum^T_{t = 1} [ \nabla R ( x_{t - 1}) - \nabla R ( x_t)] \circ u_{t-1}
 + \sum^T_{t = 1} [ \nabla R ( x_t) - \nabla R ( y_t)] \circ x_{t - 1}
  \nonumber\\
&=& \underbrace{
\sum^T_{t = 1} [ B_R ( u, x_{t - 1}) - B_R ( u, x_t)]
}_\text{P} + \underbrace{\sum^T_{t = 1} B_R ( x_{t - 1}, x_t)}_\text{B} + \nonumber\\
&&\underbrace{
\sum^T_{t = 1} [ \nabla R ( x_{t - 1}) - \nabla R ( x_t)] \circ u
}_\text{Q} + 
\underbrace{
\sum^T_{t = 1} [ \nabla R ( x_t) - \nabla R ( y_t)] \circ x_{t - 1}
}_\text{C}
\end{eqnarray}
Last equality comes from the application of 3-point equality of Bregman projection.

\begin{eqnarray}
P+Q & = & \sum^{T-1}_{t = 1} [ B_R ( u_t, x_t) - B_R ( u_{t-1}, x_t)] + \sum^{T-1}_{t = 1} [ \nabla R(x_t) ( u_t, u_{t-1})] + A \nonumber \\
& = & \sum^{T-1}_{t=1}(R( u_{t}) - R ( u_{t-1})) + A \nonumber
\end{eqnarray}

\textbf{Part 2:}
Simplifying the expression of $\mathrm{B}$ using the definition of $B_R (x,y)$, we get:
\begin{eqnarray}
{B} & \leq & \sum^T_{t = 1} [ B_R ( x_{t-1}, x_t) + B_R ( x_t, x_{t-1})] \nonumber \\
 & = & \sum^T_{t = 1} [ \nabla R ( x_{t - 1}) - \nabla R ( x_t)]
  \circ ( x_{t - 1} - x_t) \nonumber\\
  & = & \sum^T_{t = 1} [ \nabla R ( x_{t - 1}) - \nabla R ( y_t)] \circ (
  x_{t - 1} - x_t) + \sum^T_{t = 1} [ \nabla R ( y_t) - \nabla R ( x_t)] \circ
  ( x_{t - 1} - x_t) \nonumber\\
  & \leq & \sum^T_{t = 1} \| \eta c_t \|_{\star}^2 + \sum^T_{t = 1} [ \nabla R (
  y_t) - \nabla R ( x_t)] \circ ( x_{t - 1} - x_t) \nonumber
\end{eqnarray}
Last inequality follows from the application of Projection Lemma~\ref{lem:proj} and OMD update rules.

Combining $\mathrm{B}$ and $\mathrm{C}$, we get:
\begin{center}
  $\mathrm{B} + \mathrm{C} \leqslant \sum^T_{t = 1} \| \eta c_t \|_{\star}^2 - \sum^T_{t = 1} [ \nabla R ( y_t) - \nabla R ( x_t)] \circ x_t $
\end{center}

Now we appeal to the complementary slackness (or shifted slackness) condition of dual variables $a_t$(Section~\ref{subsec:2BallDualFeas}). For both $1$-norm and $2$-norm this immediately gives us the following required bound:
\begin{eqnarray*}
\mathrm{B} + \mathrm{C} & \leqslant & \sum^T_{t = 1} \| \eta c_t \|_{\star}^2 - \eta \cdot D\sum^T_{t = 1} a_{t - 1}
\end{eqnarray*}

This result can be easily extended for general $p$-norm by using similar complementary slackness criteria.
\QED
\end{proof}

\section{$0$-lookahead : Drifting Adversaries}
\label{sec:drift}

\subsection{Drifting Regret for arbitrary norm}
\label{subsec:drift_p_norm}
In this section, we study the drifting regret model when the convex body is a $p$-norm ball ($1 < p \leqslant 2$). We recall, in this model regret is measured against the optimal sequence of vectors $u^{\star}_0, u^{\star}_1, \ldots ., u^{\star}_{T-1}$ such that $\sum_{t = 1}^{T-1} \| u_t - u_{t - 1} \|_p \leq L$.

First we present the convex program for this problem. Let $q$ be the dual norm of $p$. The primal program modelling $\opt$ is presented as $\cp_5$, and its Fenchel dual is presented as $\cp_6$.

\[
\begin{array}{cll}
\cp_5: &\\[0.25em]
\min &{\displaystyle  \sum^T_{t = 1} c_t \circ x_{t - 1}}\\[1.0em]
{\rm s.t.} &{\displaystyle 
\| x_t \|_p \leqslant 1}&{\forall} t \geqslant 0 \\[0.25em]
&{\displaystyle z_{t} = x_{t} - x_{t - 1}}&{\forall} i,t \geqslant 1\\[0.25em]
&{\displaystyle \sum_{t=1}^{T-1} \|z_t\|_p \leqslant L}&\
\end{array}
\quad
\begin{array}{||cll}
\cp_6: &\\[0.25em]
\max &{\displaystyle - \sum^{T-1}_{t = 0} a_t - \alpha{\cdot}L   }\\[1.0em]
{\rm s.t.} &{\displaystyle 
 a_t  \geqslant \|b_{t + 1} - b_{t} - c_{t + 1}\|_q}& {\forall} t\geqslant 1 \\[0.25em]
&{\displaystyle a_0 
  \geqslant \|b_1 - c_1\|_q}&t=0\\[0.25em]
&{\displaystyle \|b_t\|_q \leqslant \alpha} & {\forall} t\geqslant 1\\[0.25em]
&{\displaystyle a_t \geqslant 0, \alpha \geqslant 0}&{\forall} t \geqslant 0\
\end{array}
\]

\begin{theorem}
\label{thm:drift2norm}
Consider the convex programs $\cp_5$ and $\cp_6$. 
Let $c_1, c_2, \ldots, c_T$ be a sequence of {\em linear} cost functions and $R$ is a strongly convex regularizer w.r.t $p$-norm in the $p$-ball. Then there exists a dual incrementing algorithm that generates a sequence of primal vectors $x_0, x_1, \ldots, x_{T-1}$ along with a sequence of feasible dual variables $a_t$ ($t = 0, \ldots, T-1$), $b_{i,t}$ ($i = 1, \ldots, n$,$t = 1, \ldots, T$) and $\alpha$ such that for any ${u_0,u_1,\ldots,u_{T-1} } \in K$ the following holds: \\
\begin{eqnarray*}
\eta \sum_{t=0}^{T-1} c_t \circ x_{t-1} &\leqslant & \sum_{t=1}^{T-1} \| u_t - u_{t-1} \|_p \| \nabla R(u_t) \|_q + A + \eta^2 \sum_{t=1}^{T-1} \| c_t\|^2_q - \eta \sum_{t=1}^{T-1} a_{t-1}
\end{eqnarray*}
where $A=B_R(u_0,x_0)-B_R(u_{T-1},x_T)+\nabla R(x_0)\circ u_{0} - \nabla R(x_T)\circ u_{T-1}$. 
\end{theorem}

\begin{proof}
As before our updates are guided by the OMD paradigm. (See Section 2 reference). The dual updates are exactly same as that of Section [reference]. Note that the convex program $\cp_4$ has exactly one more dual variable as compared to $\cp_2$, viz. $\alpha$. While the other dual variables are set precisely as before, $\alpha$ is set as $\alpha =\frac{D}{\eta}$. Clearly, dual feasibility is maintained. 

Since $R$ is a convex function, 
\begin{eqnarray}
  R ( u_t) - R ( u_{t - 1}) & \leqslant & \nabla R ( u_t) \circ ( u_t - u_{t - 1})
  \leqslant \| ( u_t - u_{t - 1}) \|_2 \| \nabla R ( u_t) \| \nosymbol \nosymbol_2
  \nonumber
\end{eqnarray}
where the last inequality follows from H\"{o}lder's inequality. Plugging this in Lemma ~\ref{lem:oco_base}, we get
\begin{eqnarray}
  \eta S_0 & \leq & \sum^T_{t = 1}
  \| ( u_t - u_{t - 1}) \|_2 \| \nabla R ( u_t) \| \nosymbol \nosymbol_2 + A + \eta^2 \sum_{t=1}^{T} \| c_t \|_2^2 - \eta \sum_{t=1}^T a_{t-1}
  \nonumber  
\end{eqnarray}
\QED
\end{proof}

\begin{corollary}
\label{coro:drift_2_norm}
When convex body $K$ is $2$-ball and $R(x)=\frac{1}{2}\| x\|_2$ is the strongly convex regularizer w.r.t $2$-norm, Theorem~\ref{thm:drift2norm} gives the following bound:
\begin{eqnarray}
  \sum_{t=0}^{T-1} c_t \circ x_{t-1} & \leqslant & \frac{2}{\eta}(D(L+1) + {D}^2 ) + \eta \| c_t \|_2^2 + \left( - \frac{D}{2} \cdot \sum_{t=1}^T a_{t-1} - \alpha \cdot L \right)
  \nonumber  
\end{eqnarray}
where $D= \max_{u, v \in K} \| u - v \|_2$
\end{corollary}
\begin{proof}
Note that in $\cp_5$ we have $D=2$. However as we have seen earlier, the convex program easily extends for the $D$ Diameter ball. We observe $\| \nabla R (u) \|_2 \leqslant D = \eta \cdot \alpha$ $\forall u \in K$ where $K$ is the $2$-norm ball. Also $B_R(u_0,x_0)=\frac{1}{2} \| u_0-x_0\|_2^2 \leqslant 2D^2$. So we get, ${A} \leqslant 2{D}^2 + 2D$. So combining everything, from Theorem~\ref{thm:drift2norm} we get:
\begin{eqnarray}
  \eta S_0 & \leq & 2D(L+1) + 2{D}^2 + \eta^2 \| c_t \|_2^2 + \eta \left( - \frac{D}{2} \cdot  \sum_{t=1}^T a_{t-1} - \alpha \cdot L \right)
  \nonumber  
\end{eqnarray}
\QED
\end{proof}

\subsection{Drifting Expert}
\label{subsec:drift_expert}
\begin{proof}[of Theorem~\ref{thm:drift1norm}]
In Theorem~\ref{thm:drift2norm}, plugging $p=1$ \& $q=\infty$, we derive:
\begin{eqnarray*}
\eta \sum_{t=0}^{T-1} c_t \circ x_{t-1} &\leqslant & \sum_{t=1}^{T-1} \| u_t - u_{t-1} \|_1 \| \nabla R(u_t) \|_{\infty} + A + \eta^2 \sum_{t=1}^{T-1} \| c_t\|^2_q - \eta \sum_{t=1}^{T-1} a_{t-1}
\end{eqnarray*}
Now expanding the expression for Bregman Divergence in $A$, we get
\begin{eqnarray*}
A & = & (R(u_0)-R(x_0))-(R(u_{T-1})-R(x_T))+(\nabla R(x_0)\circ x_0 - \nabla R(x_{T}) \circ x_T)
\end{eqnarray*}
It is easy to show for $R(x)=\sum_i{(x_i+\theta)\ln{(x_i+\theta)}}$ that the above quantity is upper bounded by $6 \ln{n}$. Using $\eta \cdot \alpha = \ln{n}$, and $\sum_{t=1}^{T-1} \| u_t - u_{t-1} \|_1 = 2L$ we derive the statement of the Theorem.
\end{proof}

\section{The Setting of $1$-lookahead}
\label{1lookaheadappendix}
In this section, we present the remaining proofs for the setting of $1$-lookahead.

\subsection{$1$-lookahead for the $\ell_1$-norm}
\label{1lookahead1normappendix}
In this section, we will consider the $\mts$ problem. The $\mts$ problem 
has a rich history in the context of computer science as well as learning. 

Buchbinder et al~\cite{buchbinder12} consider the $\mts$ problem and
(effectively) give a primal-dual algorithm for this problem. They then connect
up the Multiplicative Weight Update ($\mwu$) setting as one having $0$-lookahead as
compared to $1$-lookahead in the $\mts$ problem. The $\lp$ for the $\alpha$-unfair
version of the $\mts$ problem is shown in $\lp_1$; the dual is written down as $\lp_2$.

Some important conceptual differences between the $1$-lookahead and the $0$-lookahead case 
need to be mentioned. In the $0$-lookahead case, we can get {\em additive error} bounds (and this is 
the {\em regret}). In the $1$-lookahead case however, one may only achieve {\em multiplicative error} bounds (and this is the {\em competitive ratio}). Note that in order to achieve multiplicative factors, it 
is necessary that $\opt$ always has a {\em non-negative} value. In order for this to hold true, we have to assume in the foregoing that all the cost functions $c_t$ satisfy $c_{i,t}\geqslant 0$. Such a stipulation was not necessary for the $0$-lookahead case (as in Sections~\ref{sec:oco} and \ref{sec:drifting_expert}). 

\[
  \begin{array}{cll}
\lp_1:&\\[0.25em]
    \min & {\displaystyle\sum_{t = 1}^T \sum_{i = 1}^n c_{i, t} \cdot x_{i, t} }&{\displaystyle +\sum_{t
    = 1}^T \sum_{i = 1}^n \alpha \cdot z_{i, t}
}\\[1.0em]
   {\rm s.t.} &{\displaystyle \sum_{i = 1}^n x_{i, t} = 1} & {\forall} t  \geqslant 0 \\[0.25em]
    & {\displaystyle z_{i, t} \geqslant x_{i, t} - x_{i, t -
    1}} &  {\forall} i,t\geqslant 1\\[0.25em]
    & {\displaystyle z_{i, t} \geqslant 0}&{\forall} i, t \geqslant 1\\
  \end{array}
\quad
 \begin{array}{||cll}
\lp_2:&\\[0.25em]
    \max  &{\displaystyle \sum_{t = 0}^T a_t}\\[1em]
    \rm{s.t.} & {\displaystyle b_{i, t + 1} \leqslant b_{i, t} +
    c_{i, t} - a_t}&{\forall} i, t \geqslant 1\\[0.25em]
     &{\displaystyle 0 \leqslant b_{i, t} \leqslant \alpha}& {\forall} i, t \\[0.25em]
    &{\displaystyle a_0 + b_{i, 1} \leqslant 0}& \forall i, t = 0 \\
  \end{array}
\]

Given the updates in Sections~\ref{sec:oco} and \ref{sec:drifting_expert}, we will consider
updates of the following form:

\begin{center}
  $\begin{array}{lll}
    {b_{t + 1}} & { =} & {\frac{\nabla R (
    \tmmathbf{1}) - \nabla R ( x_t)}{\eta}}
  \end{array}$
\end{center}
for a specific regularizer $R$. Note the addition of the term $\nabla R (\tmmathbf{1})$: this is so that the $b_{i,t}$'s may be positive (and each $x_{i,t}$ is upper bounded by $1$). 

Note that the index of $b_t$ has changed by $1$ as compared to the earlier updates. This reflects the fact that we are considering $1$-lookahead instead of $0$-lookahead.
Also note, that given this is the
$1$-lookahead setting, we can set
\begin{eqnarray*}
  \nabla R ( y_t) & = & \nabla R ( x_{t - 1}) - \eta c_t\\
  x_t & = & \Pi ( y_t)
\end{eqnarray*}

The dual variables $a_t$ will be set via an application of Claim~\ref{claim:bregman}; apply that Claim with $y^{*} = y_t$, and if the statement of that Claim yields the value $\lambda$, then set $a_t = - \lambda$.

Given the conditions of that Claim, we will therefore have:
\[
x_{i,t} > 0 \Rightarrow a_t = [\nabla R(x_t) - \nabla R(y_t)]_i
\]

\tmtextbf{Dual Feasibility:}

In order to check dual feasibility, we will prove a few things. Since $x_t$ is
the projection of $y_t$ according to the regularizer $R$, we see from the statement above 
that all the components of the vector $[ \nabla R ( x_t) - \nabla R ( y_t)]$ where $x_{i,t} > 0$
are equal.  Now we may readily check that all of the dual constraints (but one)
are satisfied. If $x_{i,t} > 0$, then in fact the corresponding dual constraint is {\em tight}. 

The only dual constraint that has not been considered is the constraint that
$\| b_t \|_{\infty} \leqslant \alpha$. Consider the vanilla version regularizer $R_0 ( x_t) =
\sum_i x_{i, t} \cdot \log x_{i, t}$ (this is the normalized negative entropy regularizer). We see that $b_{i, t} = ( \log 1 - \log
x_{i, t}) / \eta$. Thus, the quantity $\| b_t \|_{\infty}$ may be unbounded
because some $x_{i, t}$ may reach $0$. The idea is to \tmtextit{shift} the
arguments by an amount $\theta$ so that the modified $x_{i, t}$'s are bounded
away from $0$. Given this, let us modify the vanilla version regularizer as
$R ( x_t) = \sum_i ( x_{i, t} + \theta) \cdot \log ( x_{i, t} + \theta)$.
This gives us that:
\begin{eqnarray*}
  b_{i, t} & = & ( \log ( 1 + \theta) - \log ( x_{i, t} + \theta)) / \eta
\end{eqnarray*}
In order that every $b_{i, t}$ be upper bounded by $\alpha$, it suffices that
the above expression is bounded by $\alpha$ for the case $x_{i, t} = 0$ (this
is because the log function is monotone). Thus, we get the following
stipulation on $\theta$:
\begin{eqnarray*}
  ( \log ( 1 + \theta) - \log \theta) / \eta & \leqslant & \alpha\\
  \tmop{or}, \frac{1 + \theta}{\theta} & \leqslant & e^{\eta \alpha}\\
  \tmop{or} \nocomma, \theta & \geqslant & \frac{1}{e^{\eta \alpha} - 1}
\end{eqnarray*}
We set the value $\theta = \frac{1}{e^{\eta \alpha} - 1}$, and the updates are
made according to this \tmtextit{shifted} regularizer $R$.

In the following, we will estimate the service and movement costs. It will be expedient to denote 
$\wtxt$ as the {\em shifted} point $x_{t}$:
\[
\wtxit = x_{i,t} + \theta = x_{i,t} + \frac{1}{e^{\eta\alpha} -1}
\]
Given this notation, the regularizer $R(x_t)$ may be written as $R_0(\wtxt)$.

\tmtextbf{Service Cost:}

The service cost $S$ is $\sum^T_{t=1} c_t \circ x_t$. This may be bounded as follows:
\begin{eqnarray}
  \eta S &=& \eta \nosymbol \sum_{t = 1}^T c_t \circ x_{t } \\
&& {(\because c_{i,t}\geqslant 0 \;\forall i,t)}\\
&\leqslant& \eta \sum_{t = 1}^T c_t \circ \wtxt \\
& = & 
\sum^T_{t
  = 1} [ \nabla R ( x_{t - 1}) - \nabla R ( y_t)] \circ \wtxt
\nonumber\\
  & = & \sum^T_{t = 1} [ \nabla R ( x_{t - 1}) - \nabla R ( x_t)] \circ \wtxt + \sum^T_{t = 1} [ \nabla R ( x_t) - \nabla R ( y_t)] \circ \wtxt\\
&=& \sum^T_{t=1} [\nabla R(x_{t-1}) - \nabla R(x_t)]\circ (\wtxt - u) + \sum^T_{t=1} [\nabla R(x_{t-1}) - \nabla R(x_t)]\circ u + \sum^T_{t = 1} a_t \tmmathbf{1}   \circ \wtxt\\
&=& \sum^T_{t=1}[B_R(u, x_{t-1}) - B_R(u,\wtxt) - B_R(\wtxt,x_{t-1})] +[\nabla R(x_0) - \nabla R(x_T)]\circ u +
\eta \sum^T_{t = 1} a_t \tmmathbf{1} \circ \wtxt\nonumber\\
&\leqslant&  [B_R(u,x_0) - B_R(u,x_T)] + [\nabla R(x_0) - \nabla R(x_T)]\circ u + \eta\sum^T_{t=1} a_t\|\wtxt\|_1\nonumber\\
&=& [R(x_T) - R(x_0)] + [\nabla R(x_0)\circ x_0 - \nabla R(x_T)\circ x_T] +\eta \sum^T_{t=1} a_t\|\wtxt\|_1\nonumber
\end{eqnarray}

\tmtextbf{Movement Cost:}

We will use the following $1$-dimensional {\bf Projection Lemma} for the negative entropic regularizer.
The following proof is a simplification of the proof given in Buchbinder et al.~\cite{buchbinder12}.

\begin{lemma}
Given numbers $a, b\in \mathbb{R}^{+}$, we have that
\[
(a - b) \leqs a\cdot (\ln a - \ln b).
\]
\end{lemma}
 Thus, we have that:
\begin{eqnarray*}
(\widetilde{x_{i,t}} - \widetilde{x_{i,t-1}}) &\leqs& \widetilde{x_{i,t}} \cdot [\nabla R(x_t) - \nabla R(x_{t-1})]_i\\
& = & \widetilde{x_{i,t}}\cdot{\eta}\cdot(a_t - c_{i,t})\\
& \leqs & \widetilde{x_{i,t}}\cdot{\eta}\cdot{a_t} 
\end{eqnarray*}

Let the set of coordinates $i$ where $\widetilde{x_{i,t}} \geqs \widetilde{x_{i,t-1}}$ be denoted
as $J$.
We then have that $M_t/2 = \sum_{i\in J} (\widetilde{x_{i,t}} - \widetilde{x_{i,t-1}}) \leqs \sum_{i \in J} \eta \widetilde{x_{i,t}} \cdot a_t \leqs \eta\cdot (1 +n\theta) a_t$.
This completes the bounding of the movement cost.

\subsection{Service Cost for 1-Lookahead in $\ell_2$-norm}
\label{sec:1LA2normservice}
We bound the service cost for 1LA setting where the convex body is the $2$-ball of diameter $2D$. We use the dual updates from Section~\ref{subsec:2BallDualFeas} and let $\mathcal D = \sum_{t=1}^T c_t \circ k - D\sum_{t=0}^T a_t$ be the dual objective as usual.
\begin{eqnarray}
  \eta S_1 = \eta \nosymbol \sum_{t = 1}^T c_t \circ x_{t} & = & 
\sum^T_{t=1}\eta k\circ c_t + \sum^T_{t
  = 1} [ \nabla R ( x_{t - 1} - k) - \nabla R ( y_t-k)] \circ (x_t - k)
\nonumber\\
  & = & \sum^T_{t=1}\eta k\circ c_t + \sum^T_{t = 1} [ \nabla R ( x_{t - 1}-k) - \nabla R ( x_t)-k] \circ (x_t-k) + \sum^T_{t = 1} [ \nabla R ( x_t-k) - \nabla R ( y_t-k)] \circ (x_t-k)
  \nonumber
\end{eqnarray}
We assume $u$ be any $T$ points in the $2$-ball. Further, for simplicity, let us work with the shifted vectors
$x' = x - k$ 
\begin{eqnarray}
  \eta \nosymbol S_1 - \sum^T_{t=1}\eta k\circ c_t & = & 
\underbrace{
\sum^T_{t = 1} [ \nabla R ( x'_{t - 1}) - \nabla R (
  x'_t)] \circ ( x'_{t} - u)
}_\text{A} + \nonumber\\
  &  & + 
\underbrace{
\sum^T_{t = 1} [ \nabla R ( x'_{t - 1}) - \nabla R ( x'_t)] \circ u
}_\text{B} + 
\underbrace{
\sum^T_{t = 1} [ \nabla R ( x'_t) - \nabla R ( y'_t)] \circ x'_t
}_\text{C}
  \nonumber
\end{eqnarray}
Using the $3$-point equality, we decompose $A$ further as:
\begin{eqnarray}
  A & = & \sum^T_{t = 1} [ \nabla R ( x'_{t - 1}) - \nabla R ( x'_t)] \circ
  ( x_t - u ) \nonumber\\
  & = & \sum^T_{t = 1} [ B_R ( u, x'_{t - 1}) - B_R ( u, x'_t)]
  - \sum^T_{t = 1} B_R ( x'_{t - 1}, x'_t) \nonumber\\
  & \leqslant & \underbrace{
  \sum^T_{t = 1} [ B_R ( u, x'_{t - 1}) - B_R ( u, x'_t)]
}_\text{D} \nonumber
\end{eqnarray}
The last inequality follows from the fact that Bregman divergence is always non-negative.
\begin{claim}
\label{claim:C_equals_minus_a_t}
Given the dual updates $a_t$ at each iteration, $\mathrm{C}=-D\sum^T_{t=1}\eta a_t$
\end{claim}
\begin{proof}
Using the identities (\ref{eqn:2_norm_identity}) at iteration $t$ whenever $a_t > 0$,
\begin{eqnarray*}
\mathrm{C} &=& \sum^T_{t = 1} [ \nabla R ( x'_t) - \nabla R ( y'_t)] \circ x'_t = \sum^T_{t = 1}\sum_{i=1}^n[ \nabla R ( x'_t) - \nabla R ( y'_t)]_i \cdot x'_{i,t} \\
&=& \sum^T_{t = 1}\sum_{i=1}^n(-\eta a_t\cdot \frac{x'_{i,t}}{\|x'_t\|})\cdot x'_{i,t} \\
&=& \sum^T_{t = 1}-\eta a_t\|x'_t\| \\
\end{eqnarray*}
At this point, we invoke the complementary slackness conditions explained in Section~\ref{subsec:2BallDualFeas} to claim that whenever $a_t$ is non-zero, the corresponding primal constraint $\|x'_t\|\leqslant D$ holds with equality. thus proving the claim. 
\QED \\
\end{proof}
 Both $\mathrm{B}$ and $\mathrm{D}$ telescopes. Hence, combining everything, we get
\begin{eqnarray*}
  \eta \nosymbol S_1 & \leqslant & \sum^T_{t=1}\eta k\circ c_t + [\nabla R ( x'_0) - \nabla R ( x'_T)]\circ u + B_R ( u, x'_0 ) - B_R ( u, x_T) - \eta\cdot D\sum^T_{t=1} a_t \\
  & = & \sum^T_{t=1}\eta\Delta{\mathcal D} + [\nabla R ( x'_0) - \nabla R ( x'_T)]\circ u + B_R ( u, x'_0 ) - B_R ( u, x_T)  
\end{eqnarray*}
With suitable choice of initial and final points and $\eta$, $S_1$ can be shown to have the bounds claimed in Theorem~\ref{thm:1lookahead2norm}

\begin{claim}
\label{claim:pythagorean}
Let $K$ be a $2$-ball of diameter $2D$ ($\| u\|_2 \leqslant D$ $\forall u \in K$). Let $y$ be a point (possibly outside $K$), and $x$ be the Bregman projection of $y$ on to the convex body $K$. Then for strongly convex regularizer $R(x)=\frac{1}{2}\| x\|_2^2$, the following statement is true:
\begin{eqnarray}
a_t \leqslant \| c_t\|_2, & \forall & t \geqslant 1
\end{eqnarray}
where $a_t$ and $c_t$ are respectively the dual updates and the cost vector at iteration $t$.
\end{claim}

\begin{proof}
From the Pythagorean property of Bregman Projection, we get $B_R(x,y) \leqslant B_R(u,y)$. Now for the given  regularizer $R$, $B_R(x,y)=\frac{1}{2}\| x-y\|_2^2$. Hence plugging $x=x_t$, $y=y_t$ and $u=x_{t-1}$, we derive:
\begin{eqnarray*}
\| y_t - x_t \|_2^2 \leqslant \| y_t - x_{t-1} \|_2^2 \nonumber \\
a_t \leqslant \| c_t \|_2
\end{eqnarray*}
where the last inequality holds due to non-negativity of both sides
\QED
\end{proof}

\end{document}